\documentclass{article}
\usepackage[utf8]{inputenc}
\usepackage{amsmath,amsfonts,amssymb,amsthm}
\usepackage{dsfont}
\usepackage{todonotes}
\usepackage{color}
\usepackage{algorithm}
\usepackage{comment}
\usepackage{algpseudocode}
\usepackage{algorithmicx}
\usepackage{bm}
\usepackage[makeroom]{cancel}
\usepackage{standalone}
\usepackage{subcaption}
\usepackage{cleveref}
\usetikzlibrary{bayesnet}
\usetikzlibrary{arrows,positioning}

\definecolor{sarandonga}{rgb}{0.65,0.35,0}

\definecolor{urdin}{rgb}{0.4,0,0.6}

\newcommand{\M}[0]{{\mathcal M}}
\newcommand{\R}[0]{{\mathbb R}}
\newcommand{\x}[0]{{\mathbf x}}

\newcommand{\linkStrength}[0]{{\gamma}}

\title{A first approach to closeness distributions}
\author{Jesus Cerquides\\
IIIA-CSIC, Campus UAB, Cerdanyola\\ 08193 Barcelona, Spain}

\date{}

\begin{document}

\maketitle

%%%%%%%%%%%%%%%%%%%%%%%%%%%%%% Textclass specific LaTeX commands.
\theoremstyle{plain}
\newtheorem{thm}{\protect\theoremname}
\theoremstyle{definition}
\newtheorem{defn}[thm]{\protect\definitionname}
\theoremstyle{plain}
\newtheorem{prop}[thm]{\protect\propositionname}
\theoremstyle{plain}
\newtheorem{lem}[thm]{\protect\lemmaname}
\theoremstyle{definition}
\newtheorem{example}[thm]{\protect\examplename}

\makeatother
\DeclareRobustCommand{\textgreek}[1]{\leavevmode{\greektext #1}}
\providecommand{\definitionname}{Definition}
\providecommand{\examplename}{Example}
\providecommand{\lemmaname}{Lemma}
\providecommand{\propositionname}{Proposition}
\providecommand{\theoremname}{Theorem}

\begin{abstract}
Probabilistic graphical models allow us to encode a large probability distribution as a composition of smaller ones. It is oftentimes the case that we are interested in incorporating in the model the idea that some of these smaller distributions are likely to be similar to one another. In this paper we provide an information geometric approach on how to incorporate this information and see that it allow us to reinterpret some already existing models and algorithms.
\end{abstract}

\section{Introduction}

\crefalias{prop}{proposition}

%Hierarchical models are a fundamental topic in Bayesian data analysis. Following  \cite{andrew_gelman_bayesian_2013}, 

We start by introducing a simple example to illustrate the kind of problems we are interested in solving.  Consider the problem of estimating a parameter $\theta$ using data from a small experiment and a prior distribution constructed from similar previous experiments. The specific problem description is borrowed from \cite{andrew_gelman_bayesian_2013}:

\begin{quote}
{\bf Estimating the risk of tumor in a group of rats.}

In the evaluation of drugs for possible clinical application, studies are routinely performed on rodents. For a particular study drawn from the statistical literature, suppose the immediate aim is to estimate $\theta$, the probability of tumor in a population of female laboratory rats of type ‘F344’ that receive a zero dose of the drug (a control group). The data show that 4 out of 14 rats developed endometrial stromal polyps (a kind of tumor).  Typically, the mean and standard deviation of underlying tumor risks are not available. Rather, historical data are available on previous experiments on similar groups of rats. In the rat tumor example, the historical data were in fact a set of observations of tumor incidence in 70 groups of rats (\cref{table:rats}). In the $i$th historical experiment, let the number of rats with tumors be $y_i$ and the total number of rats be $n_i$. We model the $y_i$’s as independent binomial data, given sample sizes $n_i$ and study-specific means $\theta_i$. 
\end{quote}

\begin{table}[ht]
Previous experiments:

\begin{tabular}{ c c c c c c c c c c }
0/20 
& 0/20
& 0/20
& 0/20
& 0/20
& 0/20
& 0/20
& 0/19
& 0/19
& 0/19
\\ 0/19
& 0/18
& 0/18
& 0/17
& 1/20
& 1/20
& 1/20
& 1/20
& 1/19
& 1/19
\\ 1/18
& 1/18
& 2/25
& 2/24
& 2/23
& 2/20
& 2/20
& 2/20
& 2/20
& 2/20
\\ 2/20
& 1/10
& 5/49
& 2/19
& 5/46
& 3/27
& 2/17
& 7/49
& 7/47
& 3/20
\\ 3/20
& 2/13
& 9/48
& 10/50
& 4/20
& 4/20
& 4/20
& 4/20
& 4/20
& 4/20
\\ 4/20
& 10/48
& 4/19
& 4/19
& 4/19
& 5/22
& 11/46
& 12/49
& 5/20
& 5/20
\\ 6/23
& 5/19
& 6/22
& 6/20
& 6/20
& 6/20
& 16/52
& 15/47
& 15/46
& 9/24\end{tabular}
Current experiment: 4/14
\caption{\label{table:rats} Tumor incidence in 70 historical groups of rats and in the current group of rats (from \cite{tarone_use_1982}). The table displays the values of : (number of rats with tumors)/(total number of rats).}
\end{table}

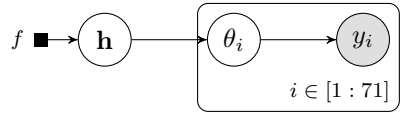
\begin{figure}
\begin{center}
\usetikzlibrary{bayesnet}
\usetikzlibrary{arrows,positioning} 
\begin{tikzpicture}[>=stealth']

  % Define nodes
  
  \node[latent]                            (h)   {$\mathbf{h}$};
  \node[latent, right=of h]            (theta_i) {$\theta_i$};
  \node[obs,    right=of theta_i]             (y_i)     {$y_i$};
  
  % Connect the nodes
  \factor[left=of h]              {f}   {left:$f$} {} {h} ;
  %\edge{f} {h};
  \edge {h} {theta_i} ; %
  \edge {theta_i} {y_i} ; %

  % Plates
  \plate {pl} {(theta_i)(y_i)} {$i\in [1:71]$} ;
  
\end{tikzpicture}

\end{center}
\caption{\label{figure:HierarchicalH}General Probabilistic graphical model for the rodents example.}
\end{figure}

\begin{figure}
\begin{center}
\usetikzlibrary{bayesnet}
\usetikzlibrary{arrows,positioning} 

\begin{tikzpicture}[>=stealth']

  % Define nodes
  \node[latent, right]            (theta_i) {$\theta_i$};
  \node[latent, left=0.6cm of theta_i]                             (a)   {$\alpha$};
  \node[latent, below=0.1cm of a]                             (b)  {$\beta$};

  \node[obs,    right=of theta_i]             (y_i)     {$y_i$};
  
  % Connect the nodes
  \factor[left=2cm of theta_i]  {f} {above:$(\alpha+\beta)^{-5/2}$} {} {a,b};
  \edge {f} {a} ; 
  \edge {f} {b} ; 
  \edge {a} {theta_i} ; 
  \edge {b} {theta_i} ; %
  \edge {theta_i} {y_i} ; %

  % Plates
  \plate {pl} {(theta_i)(y_i)} {$i\in [1:71]$} ;
  
\end{tikzpicture}
\end{center}
\caption{\label{figure:HierarchicalGelman}PGM for the rodents example proposed in \cite{andrew_gelman_bayesian_2013}.}.
\end{figure}
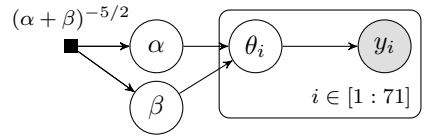

We can depict our graphical model as shown in \cref{figure:HierarchicalH}, where current and historical experiments are a random sample from a common population, having $\mathbf{h}$ as hyperparameters. Equationally our model can be described as: 
\begin{align}
    \mathbf{h} &\sim f\\
    \theta_i &\sim g(\mathbf{h})  &\forall i\in [1:71]\\
    y_i &\sim Binomial(n_i,\theta_i)& \forall i\in [1:71].
\end{align}

The model used for this problem in \cite{andrew_gelman_bayesian_2013} is the Beta-Binomial model, where  $g$ is taken to be the Beta distribution, hence $\mathbf{h} = (\alpha,\beta)$ (see \cref{figure:HierarchicalGelman}). Furthermore, in \cite{andrew_gelman_bayesian_2013} the prior $f$ over $\alpha,\beta$ is taken to be proportional to $(\alpha+\beta)^{-5/2}$, giving the model
\begin{align}
    \alpha,\beta &\propto (\alpha+\beta)^{-5/2}\\
    \theta_i &\sim Beta(\alpha,\beta)  &\forall i\in [1:71]\\
    y_i &\sim Binomial(n_i,\theta_i)& \forall i\in [1:71].
\end{align}
The presentation of the model in \cite{andrew_gelman_bayesian_2013} simply introduces the assumption that ``the Beta prior distribution with parameters $(\alpha,\beta)$ is a good description of the population distribution of the $\theta_i$’s in the historical experiments'' without further justification. In this paper we would like to show that a large part of this model can be obtained from the intuitive idea that the probability distributions for rats with tumors in each group are similar. To do that we develop a framework for encoding as a probability distribution the assumption that two probability distributions are close to each other, and rely on information geometric concepts to model the idea of closeness.
%Our objective is to provide an information geometric interpretation of the Beta-Binomial model. 

We start by introducing the general concept of closeness distribution in \cref{section:DPD}. Then, we analyze the particular case in which we choose to measure remoteness between distributions by means of the Kullback Leibler divergence in the family of multinomial distributions in \cref{sec:KL-closeness-multinomials}. The results from \cref{sec:KL-closeness-multinomials} are used in \cref{sec:Beta-Binomial-reinterpretation} to reinterpret the Beta Binomial model proposed in \cite{andrew_gelman_bayesian_2013} for the rodents example, and in \cref{sec:HDM-reinterpretation} to reinterpret the Hierarchical Dirichlet Multinomial model proposed by Azzimonti et al. in \cite{azzimonti_hierarchical_2017,azzimonti_hierarchical_2019,azzimonti_structure_2020}. We are convinced that closeness distributions could play a relevant role in probabilistic modelling, allowing for more explictly geometrically inspired probabilistic models. This paper is just a first step towards a proper definition and understanding of closeness distributions. 

\section{\label{section:DPD}Closeness distributions}
We start by introducing the formal framework required to discuss about probability distributions over probability distributions. Then, we formalize what we mean by remoteness through a remoteness function, and we introduce closeness distributions as those that implement a remoteness function. 

\subsection{Probabilities over probabilities\label{sec:ProbsOverProbs}}

Information geometry \cite{amari_information_2016} has shown us that most families of probability distributions can be understood as a Riemannian manifold. Thus, we can work with  probabilities over probabilities by defining random variables which take values in a Riemannian manifold. Here, we only introduce some fundamental definitions. For a more detailed overview of measures and probability see \cite{dudley_real_2002}, of Riemannian manifolds see \cite{jost_riemannian_2011}. Finally, Pennec provides a good overview of probability  on Riemannian manifolds in \cite{pennec_probabilities_2004}.

We start by noting that each manifold $\M$, has an associated $\sigma$-algebra, $\mathcal L_\M$, the Lebesgue $\sigma$-algebra of $\M$ (see section 1, chapter XII in \cite{amann_analysis_2009}). Furthermore, the existence of a metric $g$ induces a measure $\mu_g$ (see section 1.3 in \cite{pennec_probabilities_2004}). The volume of $\M$ is defined as $Vol(\M)=\int_\M 1 d\mu_g.$

\begin{defn}
Let $(\Omega,{\mathcal F},P)$ be a probability space and $(\M, g)$ be a Riemannian manifold. A random variable\footnote{Referred to as a random primitive in \cite{pennec_probabilities_2004}.}  $\mathbf{x}$ taking values in $\M$ is a measurable function from $\Omega$ to $\M.$
Furthermore, we say that $\mathbf{x}$ has a probability density function (p.d.f.) $p_{\mathbf{x}}$ (real, positive, and integrable function) if:
\begin{center}
$\forall {\mathcal X}\in {\mathcal L_\M} \hspace{0.3cm} P(\mathbf{x} \in {\mathcal X} )=\int_{\mathcal X}p_{\mathbf x}d\mu_g$, \hspace{0.5cm} and \hspace{0.5cm} $P(\mathcal M)=1.$
\end{center}
\end{defn}
We would like to highlight that the density function $p_\x$ is intrinsic to the manifold. If $x'=\pi(x)$ is a chart of the manifold defined almost everywhere, we obtain a random vector $\x'=\pi(\x)$. The expression of $p_\x$ in this parametrization is 
\[
p_{\x'}(x')=p_\x(\pi^{-1}(x')).
\]
Let $f:\M\rightarrow \R$ be a real function on $\M$. We define the expectation of $f$ under $\x$ as
\[
\mathbb{E}[f(\x)]=\int_{x}f(x)p_{\x}(x)d{\mu_g}
\]
We have to be careful when computing $\mathbb{E}[f(\x)]$ so that we do it independently of the parametrization. We have to use the fact that $\int_{x}f(x)p_{\x}(x)d{\mu_g}=\int_{x'}f(\pi^{-1}(x'))p_{\x'}(x')\sqrt{\mid G(x')\mid} dx', $
where $G(x')$ is the Fisher matrix at $x'$ in the parametrization $\pi$. Hence, 
\[
\mathbb{E}[f(\x)] =  \int_{x'}f(\pi^{-1}(x'))\rho_{\x'}(x') dx'.
\]
where $\rho_{\x'}(x') = p_{\x'}(x')\sqrt{\mid G(x')\mid} = p_\x(\pi^{-1}(x')) \sqrt{\mid G(x')\mid}$ is the expression of $p_\x$ in the parametrization for integration purposes, that is, its expression with respect to the Lebesgue measure $dx'$ instead of $d\mu_g.$ 

We note that $\rho_{\x'}$ depends on the chart used whereas $p_\x$ is intrinsic to the manifold.

\subsection{Formalizing remoteness and closeness}

Intuitively, the objective of this section is to create a probability distribution over pairs of probability distributions that assigns higher probability to those pairs of probability distributions which are ``close''.

We assume that we measure how distant are two points in $\M$ by means of a \emph{remoteness function} $r:\M\times \M\rightarrow \mathbb{R}$, such that  $r(x,y) \geq 0$ for each $x,y \in \M$.  Note that $r$ does not need to be transitive, symmetric or reflexive. 

As can be seen in \cref{sec:TotalOrder}, $r$ induces a total order $\leq_r$ in $\M \times \M.$ We say that two remoteness functions $r,s$ are \emph{order-equivalent} if $\leq_r=\leq_s$.

\begin{prop}\label{prop:RemotenessEquivalence}
Let $\linkStrength,\beta \in \R, \linkStrength,\beta>0$. Then, $\linkStrength\cdot r +\beta$ is order-equivalent to $r.$
\end{prop}
\begin{proof}
$a \leq_r b$  iff $r(a) \leq r(b)$ iff $\linkStrength\cdot  r(a) \leq \linkStrength \cdot r(b)$ iff $\linkStrength \cdot r(a)+\beta \leq \linkStrength \cdot r(b)+\beta$ iff $a \leq_{\linkStrength \cdot  r+\beta} b$.
\end{proof}

We say that a probability density function $p:\M\times\M\rightarrow \R$ \emph{implements} a remoteness function $r$ if $\geq_p=\leq_r.$ This is equivalent to stating that 
for each $x,y,z,t\in \M$ we have that $p(x,y) \geq p(z,t)$ iff $r(x,y) \leq r(z,t)$. That is, a density function implements a remoteness function $r$ if it assigns higher probability density to those pairs of points which are closer according to $r$.

Once we have clarified what it means for a probability to implement a remoteness function, we introduce a specific way of creating probabilities that to that. 

\begin{defn}
Let $f_r:\M\times\M\rightarrow \mathbb{R}$ be  $f_r(x,y)=\exp(-r(x,y))$. If $Z_r = \int_{\M}\int_\M f_r d\mu_g d\mu_g$ is finite, we define 
%$\mathbf{x}_r$ as the random variable taking values in $\M \times \M$, having 
the density function 
\begin{equation}
p_r(x,y)=\frac{f_r(x,y)}{Z_r}= \frac{\exp(-r(x,y))}{Z_r} \label{eq:ClosenessIntrinsic}
\end{equation}. 
We refer to the corresponding probability distribution as a closeness distribution.
\end{defn}
Note that $p_r$ is defined intrinsically. Following the explanation in the previous section, let $\pi$ be a chart of $\M$ defined almost everywhere. The representation of this pdf in the parametrization $x',y'=(\pi(x),\pi(y))$ is simply
\begin{equation}
p_{r'}(x',y')=p_r(\pi^{-1}(x'),\pi^{-1}(y'))\label{eq:ClosenessInParametrization}
\end{equation} and its representation for integration purposes is 
\begin{equation}
\rho_{r'}(x',y')=p_r(\pi^{-1}(x'),\pi^{-1}(y'))\sqrt{|G(x')|}\sqrt{|G(y')|} \label{eq:ClosenessInParametrizationForIntegration}
\end{equation}

\begin{prop}
It it exists, $p_r$ implements $r$.
\end{prop}
\begin{proof} The exponential is a monotonous function and the minus sign in the exponent is used to revert the order.
\end{proof}

\begin{prop}
\label{prop:Z_r_finite}
If $r$ is measurable and $\M$ has finite volume, then $Z_r$ is finite, and hence $p_r$ implements $r$.
\end{prop}
\begin{proof}
Note that since $r(x,y) \geq 0$, we have that $f_r(x,y) \leq 1$, and hence $f_r$ is bounded. Furthermore, $f_r$ is measurable since it is a composition of measurable functions. Now, since any bounded measurable function in a finite volume space is integrable, $Z_r$ is finite.  
\end{proof}

Obviously, once we have established a closeness distribution $p_r$ we can define its marginal and conditional distributions in the usual way. We note $p_r(x)$ (resp. $p_r(y)$) as the marginal over $x$ (resp. $y$). We note $p_r(x|y)$ (resp. $p_r(y|x)$) as the conditional density of $x$ given $y$ (resp. $y$ given $x$).

\section{\label{sec:KL-closeness-multinomials}KL-closeness distributions for multinomials}

In this section we study closeness distributions on $M_n$ (the family of multinomial distributions of dimension $n$, or the family of finite discrete distributions over $n+1$ atoms). To do that, first we need to establish the remoteness function. It is well known that there is an isometry between $M_n$ and the positive orthant of the $n$ dimensional sphere ($S_n$) of radius 2 (see section 7.4.2. in \cite{kass_geometrical_1997}). This isometry allows us to compute the volume of the manifold as the area of the sphere of radius 2 on the positive orthant.
\begin{prop}
\label{prop:Vol_M_n}
The volume of $M_n$ is $Vol(M_n)=\frac{\pi^\frac{n+1}{2}}{\Gamma(\frac{n+1}{2})}$ 
\end{prop}
\begin{proof}
The area of a sphere $S_n$ of radius $r$ is $A_{n,r} = \frac{2\pi^\frac{n+1}{2}r^n}{\Gamma(\frac{n+1}{2})}.$ Taking $r=2$, $A_{n,2} = \frac{\pi^\frac{n+1}{2} 2^{n+1}}{\Gamma(\frac{n+1}{2})}.$ Now, there are $2^{n+1}$ orthants, so the positive orthant amounts for $\frac{1}{2^{n+1}}$ of that area, as stated.
\end{proof}
\begin{figure}
    \centering
    \includegraphics[width=5cm]{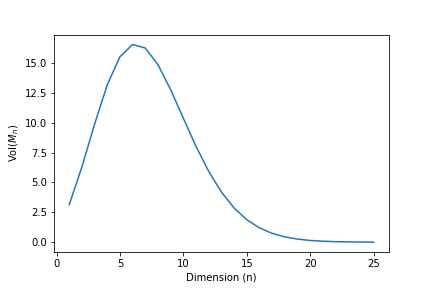}
    \caption{Volume of the family of multinomial distributions as dimension increases}
    \label{fig:Vol_M_n}
\end{figure}
\cref{fig:Vol_M_n} shows that the volume reaches its maximum at $n=7$. The main takeover of \cref{prop:Vol_M_n} is that the volume of $M_n$ is finite, because this allows us to prove the following result:
\begin{prop}
For any measurable remoteness function $r$ on $M_n$ there is a closeness distribution $p_r$ implementing it. 
\end{prop}
\begin{proof}
Directly from \cref{prop:Z_r_finite} and the fact that $M_n$ has finite volume.
\end{proof}

A reasonable choice of remoteness function for a statistical manifold is the Kullback-Leibler (KL) divergence. The next section analyzes the closeness distributions that implement KL in $\M_n$.

\subsection{Closeness distributions for KL as remoteness function}

Let $\theta \in \M_n.$ Thus, $\theta$ is a discrete distribution over $n+1$ atoms. We write $\theta_i$ to represent $p(x=i|\theta).$  Note that each $\theta_i$ is independent of the parametrization an thus it is an intrinsic quantity of the distribution.

Let $\theta,\eta \in \M_n$. The KL divergence between $\theta$ and $\eta$ is 
\[
D(\mu,\theta) = \sum_{i=1}^{n+1} \mu_i \log \frac{\mu_i}{\theta_i}.
\]

We want to study the closeness distributions that implement KL in $\M_n$. The detailed derivation of these results can be found in \cref{sec:MultinomialKLClosenessDetails}.
The closeness pdf according to \cref{eq:ClosenessIntrinsic} is 

\begin{align*}
p_D(\mu,\theta) = \frac{1}{Z_D} \prod_{i=1}^{n+1}{{\theta_i}^{\mu_i}}
        \prod_{i=1}^{n+1}{{\mu_i}^{-\mu_i}}
\end{align*}

The marginal for $\mu$ is 
\begin{align*}
p_D(\mu) =  \frac{1}{Z_D} \prod_{i=1}^{n+1}{\mu_i}^{-\mu_i} B(\mu+\frac{1}{2})
\end{align*}
where $B(\alpha)=\frac{\prod_{i=1}^k \Gamma(\alpha_i)}{\Gamma(\sum_{i=1}^k \alpha_i)}$ is the multivariate Beta function. 

And the conditional for $\theta$ given $\mu$:
\begin{align}
p_D(\theta\mid\mu) = \frac{ 
                \prod_{i=1}^{n+1}{\theta_i}^{\mu_i}
            }
            {
                B(\mu+\frac{1}{2})
            }
            %\nonumber\\
    %&= Dirichlet(\theta;\mu+\frac{1}{2}) 
    \label{eq:ConditionalWithoutTightness}
\end{align}

\Cref{eq:ConditionalWithoutTightness} is very similar to the expression of a Dirichlet distribution. In fact, the expression of $p_D(\theta\mid\mu)$ for integration purposes in the expectation parameterization , namely $\rho_D(\theta\mid\mu)$, is that of a Dirichlet distribution:

\begin{align}
\rho_D(\theta\mid\mu) =  Dirichlet(\theta;\mu+\frac{1}{2})  \label{eq:ConditionalRhoWithoutTightness}
\end{align}

\Cref{eq:ConditionalRhoWithoutTightness} deserves some attention. We have defined the joint density $p_D(\mu,\theta)$ so that pairs of distributions $(\mu,\theta)$ that are close in terms of KL divergence are assigned a higher probability than pairs of distributions $(\mu^*,\theta^*)$ which are further away in terms of KL. Hence, the conditional $p_D(\theta\mid\mu)$ assigns a larger probability to those distributions $\theta$ which are close in terms of KL to $\mu.$ This means that whenever we have a probabilistic model which encodes two multinomial distributions $\theta$ and $\mu$, and we are interested in introducing that $\theta$ should be close to $\mu$, we can introduce the assumption that $\theta \sim Dirichlet(\mu+\frac{1}{2}).$ 

Interesting as it is for modeling purposes, the use of  \cref{eq:ConditionalRhoWithoutTightness} however does not allow the modeler to convey information regarding the strength of the link. That is, $\theta$'s in the KL-surrounding of $\mu$ will be more probable, but there is no way to establish how much more probable. We know by \cref{prop:RemotenessEquivalence} that for any remoteness function $r$, we can select $\linkStrength,\beta>0$, and $\linkStrength\cdot r+\beta$ is order-equivalent to $r$. We can take advantage of that fact and use $\linkStrength$ to encode the strength of the probabilistic link between $\theta$ and $\mu$. If instead of using the KL ($D$) as remoteness function, we opt for $\linkStrength \cdot D$, following a parallel development to the one above we will find that  
\begin{equation}
    \rho_{\linkStrength \cdot D}(\theta\mid\mu) = Dirichlet(\theta;\linkStrength\mu+\frac{1}{2}).\label{eq:ConditionalWithTightness}
\end{equation}

Now, \cref{eq:ConditionalWithTightness} allows the modeler to fix a large value of $\linkStrength$ to encode that it is extremely unlikely that  $\theta$ separates from $\mu$, or a value of $\linkStrength$ close to $0$ to encode that the link between $\theta$ and $\mu$ is highly loose. Furthermore it is important to realize that \cref{eq:ConditionalWithTightness} allow us to interpret any already existing model which incorporates Dirichlet (or Beta) distributions with the only requirement that each of its concentration parameters is larger than $\frac{1}{2}.$ Say we have a model in which $\theta \sim Dirichlet(\alpha)$. 
Then, defining $\mu$ by coordinates as $\mu_i=\frac{\alpha_i-\frac{1}{2}}{-\frac{n+1}{2}+\sum_{i=1}^{n+1}\alpha_i}$, we can interpret the model as imposing $\theta$ to be close to $\mu$ with intensity $\linkStrength=\frac{\alpha_1-\frac{1}{2}}{\mu_1}.$ Note that, extending this interpretation a bit to the extreme, since the strength of the link reduces as $\linkStrength\rightarrow 0$, a "free" Dirichlet will have all of its weights set to $\frac{1}{2}.$ This coincides with the classical prior suggested by Jeffreys \cite{jeffreys_invariant_1946,jeffreys_theory_1998} for this very same problem. This is reasonable since Jeffreys' prior was constructed to be independent of the parametrization, that is, to be intrinsical to the manifold, similarly to what we are doing.

\subsection{Visualizing the distributions}
In the previous section we have seen provided an expression for $p_{\linkStrength \cdot D}(\theta\mid\mu).$ Since the KL divergence is not symmetric, we have that $p_{\linkStrength \cdot D}(\mu\mid\theta)$ is different from $p_{\linkStrength \cdot D}(\theta\mid\mu).$ Unfortunately, we have not been able to provide a closed form expression for $p_{\linkStrength \cdot D}(\mu\mid\theta).$ However, it is possible to compute it numerically in order to compare both conditionals.  
\begin{figure}
    \centering
    \includegraphics[scale=0.3]{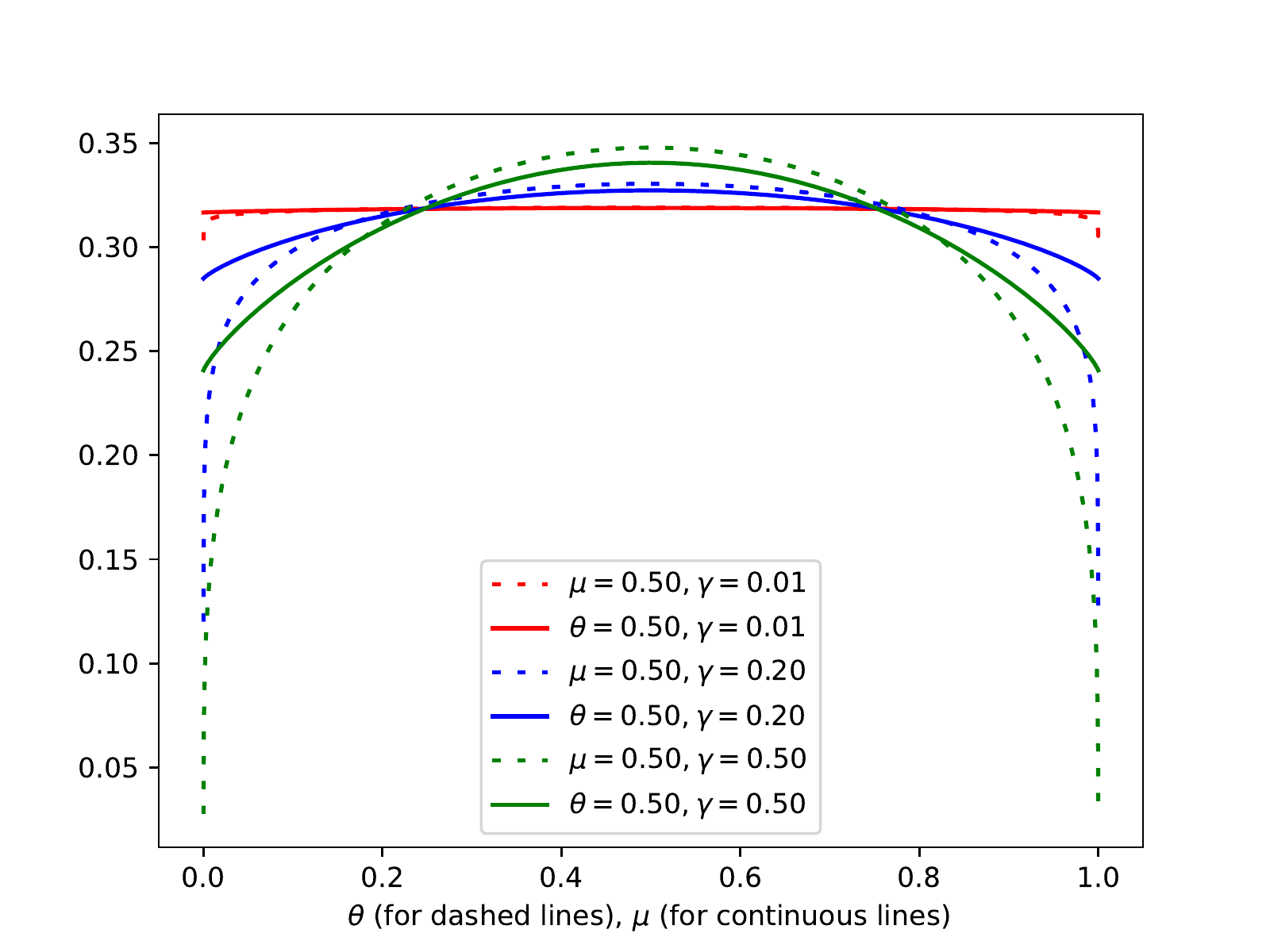}
    \includegraphics[scale=0.3]{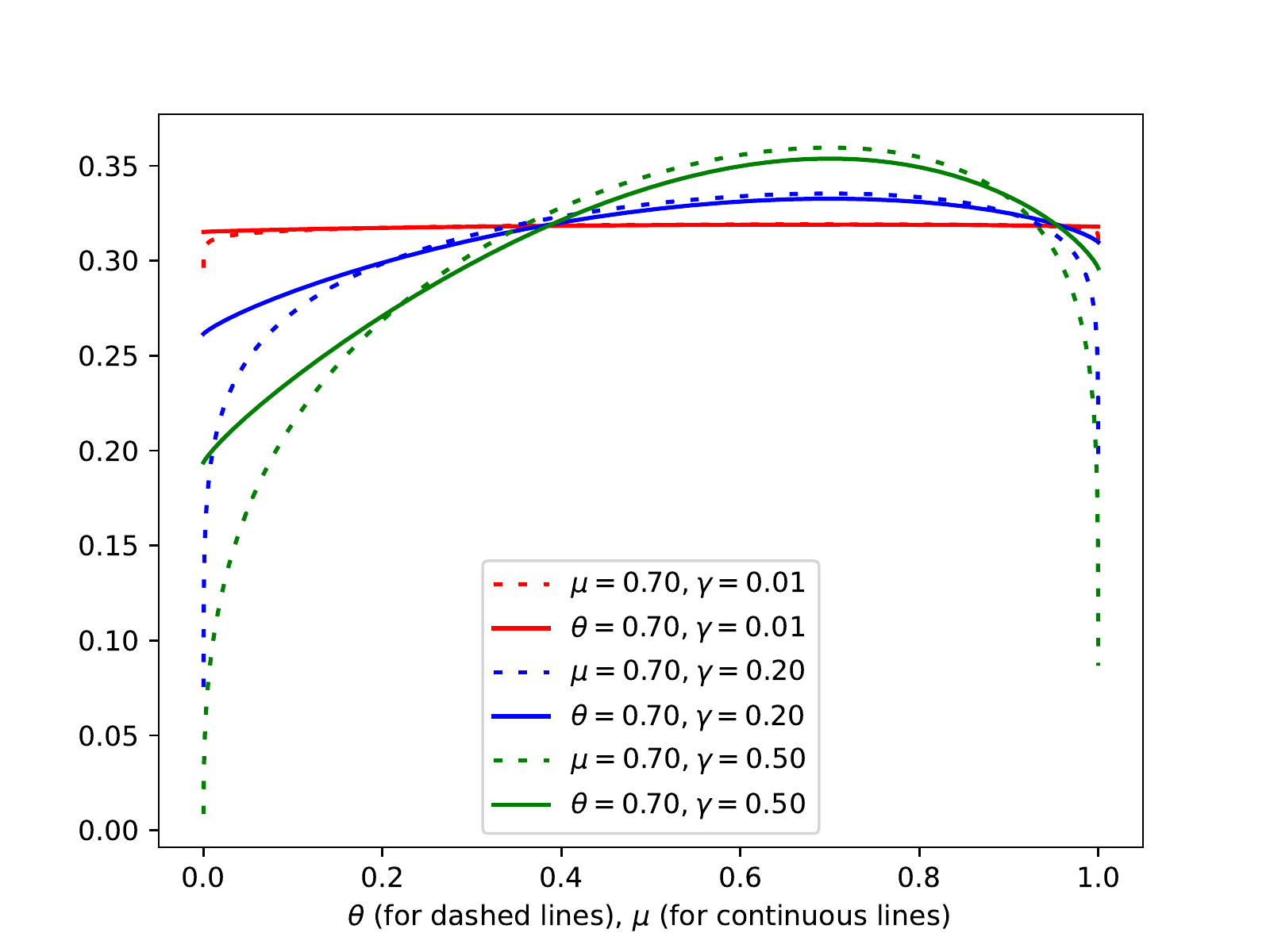}
    \includegraphics[scale=0.3]{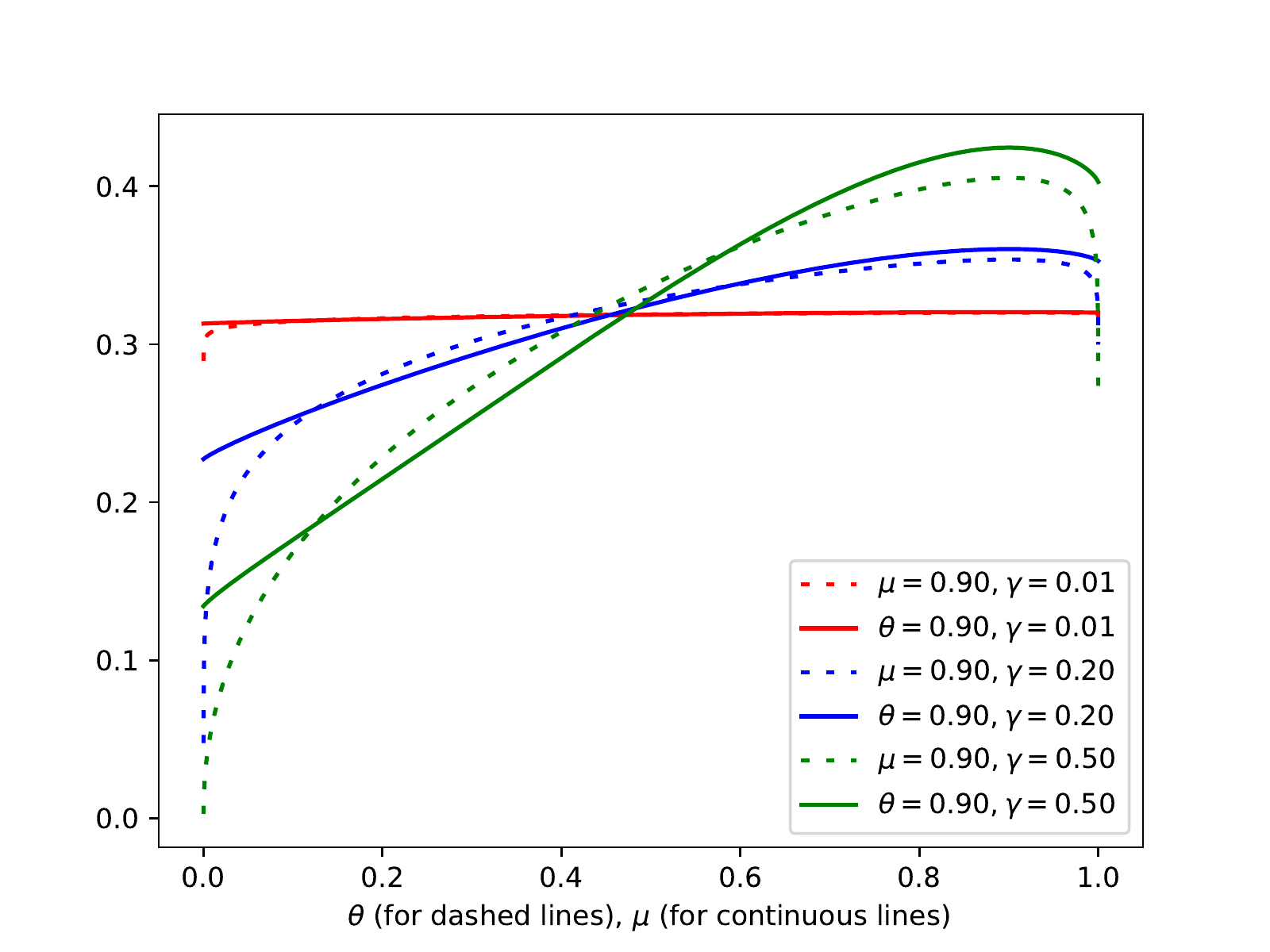}
    \includegraphics[scale=0.3]{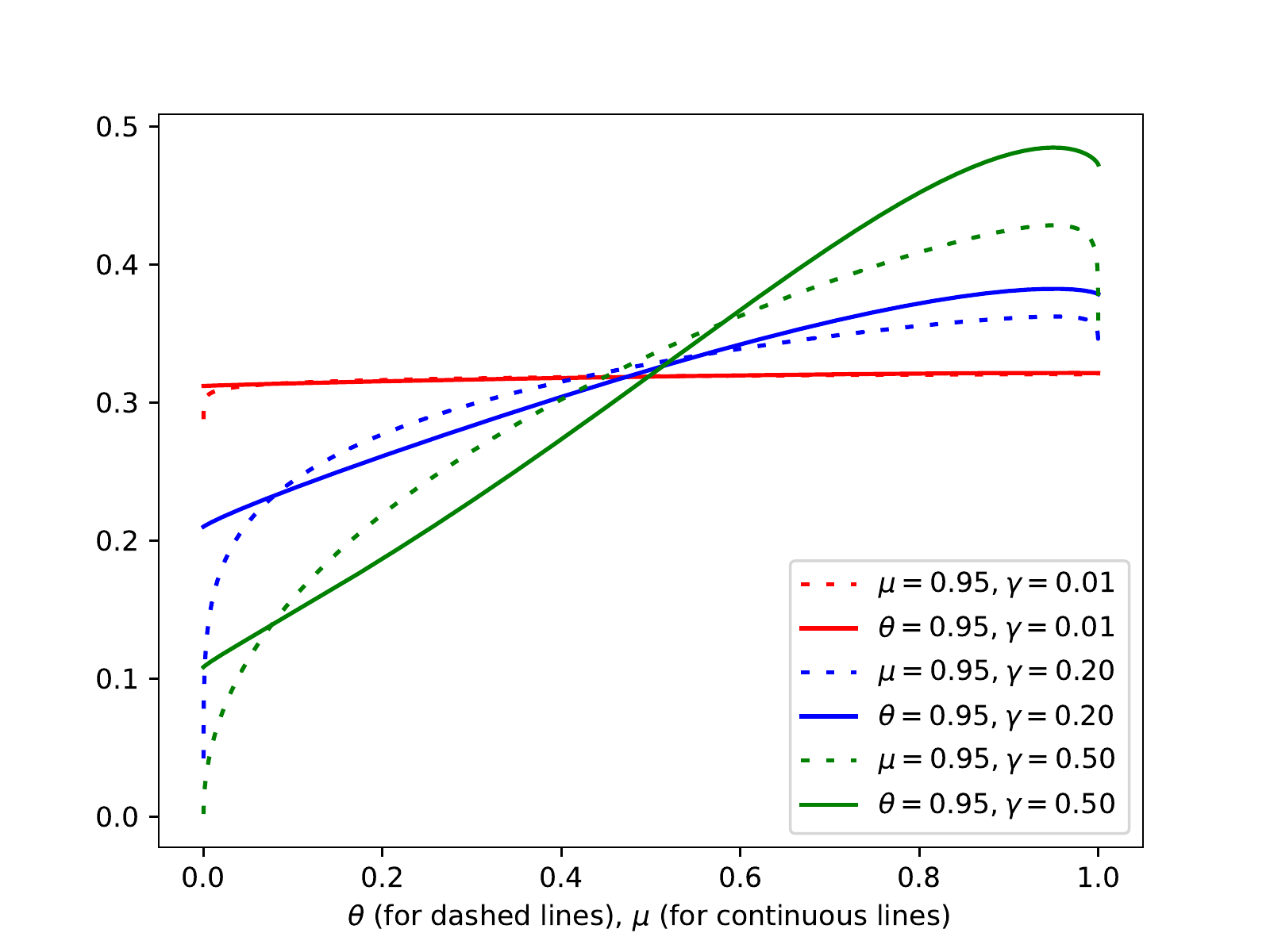}
    \caption{Comparison of $p_{\linkStrength \cdot D}(\theta\mid\mu)$ and $p_{\linkStrength \cdot D}(\theta\mid\mu).$}
    \label{fig:ConditionalComparison}
\end{figure}

\Cref{fig:ConditionalComparison} shows a comparison\footnote{According to what is suggested in \cite{cerquides2021parametrization}, for a proper interpretation of the densities, we show its density function, which is intrinsic to the manifold, instead of its expression in the parametrization, as is commonly done.}  of  $p_{\linkStrength \cdot D}(\mu\mid\theta)$ and $p_{\linkStrength \cdot D}(\theta\mid\mu).$ Note that from \cref{eq:ConditionalWithTightness}, the value of $p_{\linkStrength \cdot D}(\theta\mid\mu)$ is $0$ at $\theta=0$ and $\theta=1.$  In \cref{fig:ConditionalComparison}, we can see that this is not the case for $p_{\linkStrength \cdot D}(\mu\mid\theta)$ neither at $\mu=0$ nor at $\mu=1$. In fact we see that $p_{\linkStrength \cdot D}(\theta\mid\mu)$ always starts below $p_{\linkStrength \cdot D}(\mu\mid\theta)$ at $\theta=0$ (resp. $\mu=0)$. Then, as $\theta$ (resp. $\mu$) grows, it is always the case that  $p_{\linkStrength \cdot D}(\theta\mid\mu)$ goes over $p_{\linkStrength \cdot D}(\mu\mid\theta)$, to end decreasing again below it when $\theta$ (resp $\mu$) approaches to 1.

\begin{comment}
\subsection{Closeness distributions for the Fisher distance}

\todo[inline]{This has to be worked, but if algebraically works we could define a better alternative to the Dirichlet distribution.}

The isometry between $M_n$ and $S_n$ allows us to provide a precise expression of the distance between two $n$-dimensional multinomials. If we parametrize the $n$-dimension multinomials using the $n+1$ proportions $\theta_1,\ldots,\theta_{n+1}$ we have that:
\[
d(\theta,\theta^*)=2\cdot\arccos\{ \sum_{i=1}^{n+1}\sqrt{\theta_i \theta_i^*} \}
\]
We could take the 
\[
exp(-d(\theta, \theta^*)) = \sum_{i=1}^{n+1}\sqrt{\theta_i \theta_i^*}
\]

\end{comment}

\section{\label{sec:Beta-Binomial-reinterpretation}Reinterpreting the Beta-Binomial model}

\begin{figure}
\begin{center}
\usetikzlibrary{bayesnet}
\usetikzlibrary{arrows,positioning} 
\begin{tikzpicture}[>=stealth']

  % Define nodes
  \node[latent]                            (mu)   {$\mu$};
  \node[latent, below=0.3cm of mu]                (gamma)   {$\linkStrength$};
  \node[latent, below right=0.1cm and 0.6cm of mu]            (theta_i) {$\theta_i$};
  \node[obs,    right=of theta_i]             (y_i)     {$y_i$};
  
  % Connect the nodes
  \factor[left=0.8cm of mu] {f_mu} {above:$Beta(\frac{1}{2},\frac{1}{2})$} {} {mu};
  \factor[left=0.8cm of gamma] {f_gamma} {above:$Gamma(1,0.1)$} {} {gamma};
  \edge {mu,gamma} {theta_i} ; %
  \edge {theta_i} {y_i} ; %

  % Plates
  \plate {pl} {(theta_i)(y_i)} {$i\in [1:71]$} ;
  
\end{tikzpicture}
\end{center}
\caption{\label{figure:HierarchicalMu}Reinterpreted hierarchical graphical model for the rodents example.}
\end{figure}
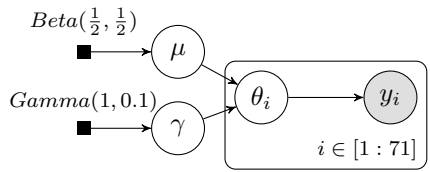
We are now ready to go back to the rodents example provided in the introduction. The main idea we would like this hierarchical model to capture is that the $\theta_i$'s are somewhat similar. We do this by introducing a new random variable $\mu$ to which we would like each $\theta_i$ to be close to (see \cref{figure:HierarchicalMu}). Furthermore, we introduce another variable $\linkStrength$ that controls how tightly coupled the $\theta_i$'s are to $\mu$. Now, $\mu$ represents a proportion and priors for proportions have been well studied, including the "Bayes-Laplace rule" \cite{laplace_essai_1814} which recommends $Beta(1,1)$, the Haldane prior \cite{haldane_note_1932} which is an improper prior $\lim_{\alpha\rightarrow 0^+}Beta(\alpha,\alpha)$, and the Jeffreys' prior \cite{jeffreys_invariant_1946,jeffreys_theory_1998} $Beta(\frac{1}{2},\frac{1}{2}).$ Following the arguments in the previous section, here we stick with the Jeffreys' prior. A more difficult problem is the selection of the prior for $\linkStrength$, where we still do not have a well founded choice. Note that taking a look at \cref{eq:ConditionalWithTightness}, $\linkStrength$'s role acts similarly (although not exactly equal) to an equivalent sample size. Thus, the prior over $\linkStrength$ could be thought as a prior over the equivalent sample size with which $\mu$  will be incorporated as prior into the determination of each of the $\theta_i$'s. In case the size of each sample ($n_i$) is large, there will be no much difference between a hierarchical model and modeling each of the 71 experiments as independent experiments. So, it makes sense for the prior over $\linkStrength$ to concentrate on relatively small equivalent sample sizes. Following this line of thought we propose $\linkStrength$ to follow a $Gamma(\alpha=1,\beta=0.1).$ 

To summarize, the hierarchical model we obtain based on divergence probability distributions is:
\begin{align}
    \mu & \sim Beta(\frac{1}{2},\frac{1}{2}) \\ 
    \linkStrength & \sim Gamma(1,0.1)\\
    \theta_i &\sim Beta(\linkStrength\mu+\frac{1}{2},\linkStrength(1-\mu)+\frac{1}{2})  &\forall i\in [1:71]\\
    y_i &\sim Binomial(n_i,\theta_i)& \forall i\in [1:71].
\end{align}

\begin{figure}
    \centering
    \includegraphics[width=\textwidth]{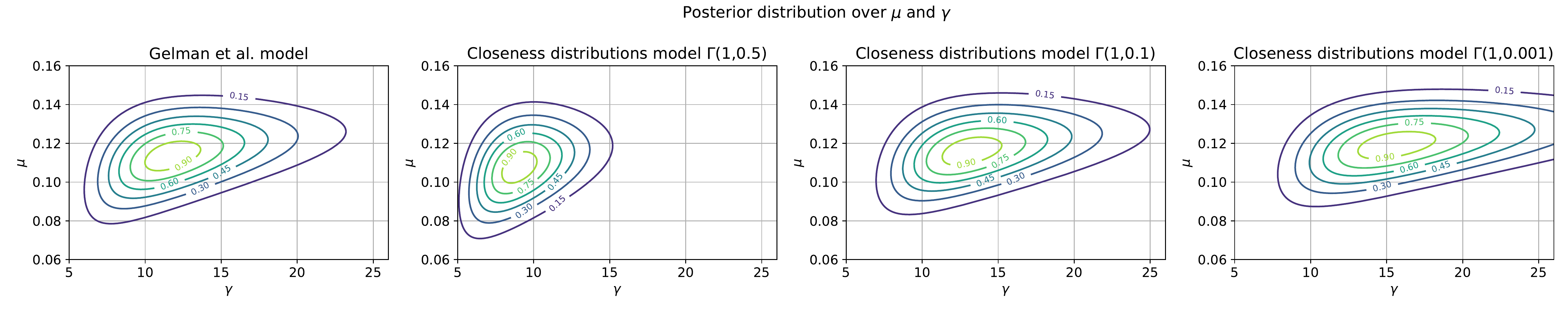}
    \caption{Comparison of posteriors between a closeness distribution model and that proposed by Gelman et al. in }%\cite{andrew_gelman_bayesian_2013}.}
    \label{fig:BDA3PosteriorComparison}
\end{figure}
\Cref{fig:BDA3PosteriorComparison} shows that the posteriors generated by both models are similar, and put the parameter $\mu$ (the pooled average) between 0.08 and 0.15 and the parameter $\gamma$ (the intensity of the link between $\mu$ and each of the $\theta_i$'s) between 5 and 25. Furthermore, the model is relatively insensitive to the parameters of the prior for $\linkStrength$ as long as they do create a sparse prior. Thus, we see that selecting the prior as $\Gamma(1,0.5)$ creates a prior too much concentrated on low values of $\linkStrength$ (that is it imposes a relatively mild closeness link between $\mu$ and each of the $\theta_i$'s). This changes a lot the estimation. However, $\Gamma(1,0.01)$  creates a posterior similar to that of $\Gamma(1,0.1)$, despite being more spread.

\section{\label{sec:HDM-reinterpretation}Hierarchical Dirichlet Multinomial model}

Recently \cite{azzimonti_hierarchical_2017,azzimonti_hierarchical_2019,azzimonti_structure_2020}, Azzimonti et al. have proposed a hierarchical Dirichlet multinomial model to estimate conditional probability tables (CPTs) in Bayesian networks. Given two discrete finite random variables $X$ (over domain $\mathcal X$) and $Y$ over domain ($\mathcal Y$) which are part of a Bayesian network, and such that $Y$ is the only parent of $X$ in the network, the CPT for $X$ is responsible of storing $p(X|Y).$ The usual CPT model (the so called Multinomial-Dirichlet adheres to \emph{parameter independence} and stores $|{\mathcal Y}|$ different independent Dirichlet distributions over each of the $\theta_{X|y}$. Instead Azzimonti et al. propose the hierarchical Multinomial-Dirichlet model, where  ``the parameters of different conditional distributions belonging to the same CPT are drawn from a common higher-level
distribution''. 
Their model can be summarized equationally as
\begin{align*}
    \alpha & \sim  s \cdot Dirichlet (\alpha_0)\\
    \theta_{X|y} & \sim  Dirichlet(\alpha)  &\forall y \in {\mathcal Y}\\
    X_y & \sim  Categorical(\theta_{X|y}) & \forall y \in {\mathcal Y}\\
\end{align*}
and graphically as shown in \cref{figure:HierarchicalAzzimonti}. 

The fact that the Dirichlet distribution is the conditional of a closeness distribution allow us to think about this model as a generalization of the model presented for the rats example. Thus, the  Hierarchical Dirichlet Multinomial model can be understood as  introducing the assumption that there is a probability distribution with parameter $\mu$, that is close in terms of its KL divergence to each of the $y\in {\mathcal Y}$ different distributions each of them parameterized by $\theta_{X|y}.$  Thus, in equational terms, we have that the model can be rewritten as

\begin{align}
    \mu & \sim Dirichlet(\frac{1}{2},\ldots,\frac{1}{2}) \\ 
    \linkStrength & \sim Gamma(1,0.1)\\
    \theta_{X|y} &\sim Dirichlet(\linkStrength\mu+ \frac{1}{2})  &\forall y \in {\mathcal Y}\\
    X_y &\sim Categorical(\theta_{X|y}) & \forall y \in {\mathcal Y}\\
\end{align}

Note that $\linkStrength$ in our reinterpreted model plays a role quite similar to the one that $s$ played on Azzimonti's model. To maintain the parallel with the model developed for the rodents example, here we have also assumed a $Gamma(1,0.1)$ as prior over $\linkStrength$, instead of the punctual distribution assumed in \cite{azzimonti_hierarchical_2019}, but we could easily mimic their approach and specify a single value for $\linkStrength.$

\begin{figure}
\begin{center}
\usetikzlibrary{bayesnet}
\usetikzlibrary{arrows,positioning} 

\begin{tikzpicture}[>=stealth']

  % Define nodes
  \node[latent, right]            (theta_i) {$\theta_{X|y}$};
  \node[latent, left=0.6cm of theta_i]                             (a)   {$\alpha$};
  
  \node[obs,    right=of theta_i]             (y_i)     {$X_y$};
  
  % Connect the nodes
  \factor[left=2.5cm of theta_i]  {f} {above:$s\cdot Dirichlet(\alpha_0)$} {} {a};
  \edge {a} {theta_i} ; %
  \edge {theta_i} {y_i} ; %

  % Plates
  \plate {pl} {(theta_i)(y_i)} {$y\in {\mathcal Y}$} ;
  
\end{tikzpicture}
\end{center}
\caption{\label{figure:HierarchicalAzzimonti}PGM for the hierarchical Dirichlet Multinomial model proposed in \cite{azzimonti_hierarchical_2019}.}
\end{figure}
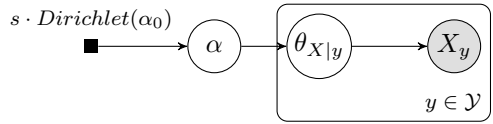

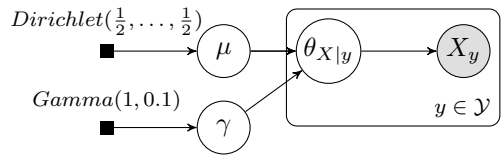
\begin{figure}
\begin{center}
\usetikzlibrary{bayesnet}
\usetikzlibrary{arrows,positioning} 

\begin{tikzpicture}[>=stealth']

  % Define nodes
  \node[latent, right]            (theta_i) {$\theta_{X|y}$};
  \node[latent, left=0.6cm of theta_i]                             (mu)   {$\mu$};
  \node[latent, below=0.3cm of mu]                (gamma)   {$\linkStrength$};
  \node[obs,    right=of theta_i]             (y_i)     {$X_y$};
  
  % Connect the nodes
  \factor[left=1.1cm of mu]  {f} {above:$Dirichlet(\frac{1}{2},\ldots,\frac{1}{2})$} {} {a};
  \factor[left=1.1cm of gamma]  {g} {above:$Gamma(1,0.1)$} {} {gamma};
  \edge {mu,gamma} {theta_i} ; %
  \edge {a} {theta_i} ; %
  \edge {theta_i} {y_i} ; %

  % Plates
  \plate {pl} {(theta_i)(y_i)} {$y\in {\mathcal Y}$} ;
  
\end{tikzpicture}
\end{center}
\caption{\label{figure:HDMReinterpreted}Reinterpreted PGM for the hierarchical Dirichlet Multinomial model }
\end{figure}

Note that we are not claiming that we are improving the Hierarchical Dirichlet Multinomial model, we are just reinterpreting it in a conceptually easier to understand way.

\section{Conclusions and future work}
We have introduced the idea of divergence distributions and we have shown that they can be a useful tool for the probabilistic model builder. We have seen that they can provide additional rationale and geometric intuitions for some commonly used hierarchical models. In this paper we have concentrated on discrete divergence distributions. The study of continuous divergence distributions remains as future work. 

\section{Acknowledgements} 
Thanks to Borja Sánchez López, Jerónimo Hernández-González and Oguz Mulayim for discussions on preliminary versions. This work was partially supported by the projects Crowd4SDG and Humane-AI-net, which have received funding from the European Union’s Horizon 2020 research and innovation program under grant agreements No 872944 and No 952026, respectively. This work was also partially supported by Grant PID2019-104156GB-I00 funded by MCIN/AEI/10.13039/501100011033.

\appendix

\section{\label{sec:TotalOrder}Total order induced by a function}

\begin{defn} Let $Z$ be a set and $f:Z\rightarrow\R$ a function. 
The binary relation $\leq_{f}$ (a subset of $Z\times Z$) is defined as 
\begin{equation}
a \leq_{f} b \text{ iff } f(a)\leq f(b)
\end{equation}
\end{defn}
\begin{prop}
$\leq_f$ is a total (or lineal) order in $Z$.
\end{prop}
\begin{proof}
Reflexivity, transitivity, antisimmetry and totality are inherited from the fact that $\leq$ is a total order in $Z$.  
\end{proof}

\section{\label{sec:MultinomialKLClosenessDetails} Detailed derivation of the KL based closeness distributions for multinomials}

\subsection{Closeness distributions for KL as remoteness function}

Let $\theta \in \M_n.$ Thus, $\theta$ is a discrete distribution over $n+1$ atoms. We write $\theta_i$ to represent $p(x=i|\theta).$  Note that each $\theta_i$ is independent of the parametrization an thus it is an intrinsic quantity of the distribution.

Let $\theta,\eta \in \M_n$. The KL divergence between $\theta$ and $\eta$ is 
\[
D(\mu,\theta) = \sum_{i=1}^{n+1} \mu_i \log \frac{\mu_i}{\theta_i}.
\]

The closeness pdf according to \cref{eq:ClosenessIntrinsic} is 

\begin{align*}
p_D(\mu,\theta)  
    &= \frac{1}{Z_D} \exp(- D(\mu,\theta)) \\
    &= \frac{1}{Z_D} \exp(- \sum_{i=1}^{n+1} \mu_i \log \frac{\mu_i}{\theta_i})\\
    &= \frac{1}{Z_D} \prod_{i=1}^{n+1}{\theta_i}^{\mu_i}
        \prod_{i=1}^{n+1}{\mu_i}^{-\mu_i} \\
\end{align*}

Now, it is possible to assess the marginal for $\mu$
\begin{align}
p_D(\mu) &= \int_\theta p_D(\mu,\theta) d\mu_g \nonumber\\
            &=  \int_\theta \frac{1}{Z_D} 
                    \prod_{i=1}^{n+1}{\theta_i}^{\mu_i} 
                    \prod_{i=1}^{n+1}{\mu_i}^{-\mu_i} d\mu_g \nonumber\\
            &=  \frac{1}{Z_D} \prod_{i=1}^{n+1}{\mu_i}^{-\mu_i} 
                    \int_\theta \prod_{i=1}^{n+1}{\theta_i}^{\mu_i} 
                     d\mu_g \label{eq:mu_marginal_first}\\
            %&=  \frac{1}{Z_D} \prod_{i=1}^{n+1}{\mu_i}^{-\mu_i} 
                    %B(\mu+\frac{1}{2})
\end{align}
To continue, we need to compute $\int_\theta \prod_{i=1}^{n+1}{\theta_i}^{\mu_i} d\mu_g$ as an intrinsic quantity of the manifold, that is, invariant to changes in parametrization. We are integrating $f(\theta) = \prod_{i=1}^{n+1}{\theta_i}^{\mu_i}$. We can parameterize the manifold using $\theta$ itself (the expectation parameters). In this parameterization the integral can be written as 
\begin{align}
\int_\theta \prod_{i=1}^{n+1}{\theta_i}^{\mu_i} d\mu_g 
    &= \int_\theta \prod_{i=1}^{n+1}{\theta_i}^{\mu_i} \sqrt{\mid G(\theta)\mid} d\theta\nonumber\\
    &= \int_\theta \prod_{i=1}^{n+1}{\theta_i}^{\mu_i} \prod_{i=1}^{n+1}{\theta_i}^{-\frac{1}{2}} d\theta\nonumber\\
    &= \int_\theta \prod_{i=1}^{n+1}{\theta_i}^{\mu_i-\frac{1}{2}}  d\theta\nonumber\\
    &= B(\mu+\frac{1}{2})\label{eq:beta_integral}, 
\end{align}
where the last equality comes from identifying it as a Dirichlet integral of type 1 (see 15-08 in \cite{jeffreys_methods_1950}), and $B(\alpha)=\frac{\prod_{i=1}^k \Gamma(\alpha_i)}{\Gamma(\sum_{i=1}^k \alpha_i)}$ is the multivariate Beta function. 

Combining \cref{eq:mu_marginal_first} with \cref{eq:beta_integral} we get 
\begin{align}
p_D(\mu) &=  \frac{1}{Z_D} \prod_{i=1}^{n+1}{\mu_i}^{-\mu_i} 
                    B(\mu+\frac{1}{2})
\end{align}

From here, we can compute the conditional for $\theta$ given $\mu$:
\begin{align}
p_D(\theta\mid\mu)
    &= \frac{ p_D(\mu,\theta) }{p_D(\mu)} \nonumber\\
    &= \frac{
                \frac{1}{Z_D} 
                \prod_{i=1}^{n+1}{\theta_i}^{\mu_i}
                \prod_{i=1}^{n+1}{\mu_i}^{-\mu_i}
            } 
            {   
                \frac{1}{Z_D} 
                \prod_{i=1}^{n+1}{\mu_i}^{-\mu_i} 
                B(\mu+\frac{1}{2})
            } \nonumber\\
    &= \frac{ 
                \prod_{i=1}^{n+1}{\theta_i}^{\mu_i}
            }
            {
                B(\mu+\frac{1}{2})
            }
            %\nonumber\\
    %&= Dirichlet(\theta;\mu+\frac{1}{2}) 
    \label{eq:ConditionalWithoutTightnessAppendix}
\end{align}

\Cref{eq:ConditionalWithoutTightnessAppendix} is very similar to the expression of a Dirichlet distribution. In fact, the expression of $\rho_D(\theta\mid\mu)$ in the expectation parameterization is that of a Dirichlet distribution:

\begin{align}
\rho_D(\theta\mid\mu)
    &= p_D(\theta\mid\mu)\sqrt{\mid G(\theta) \mid}\notag\\
    &= \frac{ 
                \prod_{i=1}^{n+1}{\theta_i}^{\mu_i}
            }
            {
                B(\mu+\frac{1}{2})
            }
        \prod_{i=1}^{n+1}{\theta_i}^{-\frac{1}{2}}\notag\\
    &= \frac{ 
                \prod_{i=1}^{n+1}{\theta_i}^{\mu_i-\frac{1}{2}}
            }
            {
                B(\mu+\frac{1}{2})
            }\notag\\
    &= Dirichlet(\theta;\mu+\frac{1}{2})  \label{eq:ConditionalRhoWithoutTightnessAppendix}
\end{align}

%%%%%%%%%%% The bibliography starts:
\typeout{}
\bibliography{references}
%\hbadness=5000
\bibliographystyle{plain}

\end{document}